\documentclass[sigconf,natbib=true]{acmart}
\AtBeginDocument{%
  }

\setcopyright{acmlicensed}
\copyrightyear{2018}
\acmYear{2018}
\acmDOI{XXXXXXX.XXXXXXX}

\acmConference[Conference acronym 'XX]{Make sure to enter the correct
  conference title from your rights confirmation email}{June 03--05,
  2018}{Woodstock, NY}
\acmISBN{978-1-4503-XXXX-X/2018/06}

\definecolor{newtext}{rgb}{0, 0, 0}

\usepackage{multirow}
\usepackage{xcolor,colortbl}
\usepackage[ruled,linesnumbered]{algorithm2e} 
\usepackage{algpseudocode}                    
\usepackage{arydshln}
\usepackage{amsthm}  
\usepackage{amsmath} 
\definecolor{darkgreen}{rgb}{0.0, 0.5, 0.0}
\theoremstyle{plain}      
\newtheorem{theorem}{Theorem}[section]

\theoremstyle{definition} 
\newtheorem{definition}[theorem]{Definition}

\definecolor{high}{RGB}{255, 75, 75}    
\definecolor{low}{RGB}{32, 178, 170}   

\begin{document}

\title{ShuffleGate: Scalable Feature Optimization for Recommender Systems via Batch-wise Sensitivity Learning}

\author{Yihong Huang}
\affiliation{%
  \institution{Bilibili Inc.}
  \city{Shanghai}
  \country{China}
}
\email{hyh957947142@gmail.com}

\author{Chen Chu}
\authornote{Corresponding author.}
\affiliation{%
  \institution{Bilibili Inc.}
  \city{Shanghai}
  \country{China}
}
\email{chuchen.blueblues@gmail.com}

\author{Fan Zhang}
\affiliation{%
  \institution{Guangzhou University}
 \city{Guangzhou}
  \country{China}}
\email{fanzhang.cs@gmail.com}

\author{Liping Wang}
\affiliation{%
  \institution{East China Normal University}
 \city{Shanghai}
  \country{China}}
\email{lipingwang@sei.ecnu.edu.cn}

\author{Fei Chen}
\affiliation{%
  \institution{Bilibili Inc.}
  \city{Shanghai}
  \country{China}
}
\email{chenfei03@bilibili.com}

\author{Yu Lin}
\affiliation{%
  \institution{Bilibili Inc.}
  \city{Shanghai}
  \country{China}
}
\email{linyu03@bilibili.com}

\author{Ruiduan Li}
\affiliation{%
  \institution{Bilibili Inc.}
  \city{Shanghai}
  \country{China}
}
\email{ruidli1992@gmail.com}

\author{Zhihao Li}
\affiliation{%
  \institution{Bilibili Inc.}
  \city{Shanghai}
  \country{China}
}
\email{zhihao.lee@foxmail.com}

\renewcommand{\shortauthors}{Trovato et al.}


\begin{abstract} Feature optimization—specifically Feature Selection (FS) and Dimension Selection (DS)—is critical for the efficiency and generalization of large-scale recommender systems. While conceptually related, these tasks are typically tackled with isolated solutions that often suffer from ambiguous importance scores or prohibitive computational costs.

In this paper, we propose \textbf{ShuffleGate}, a unified and interpretable mechanism that estimates component importance by measuring the model's sensitivity to information loss. Unlike conventional gating that learns relative weights, ShuffleGate introduces a batch-wise shuffling strategy to effectively ``erase'' information in an end-to-end differentiable manner. This paradigm shift yields naturally polarized importance distributions, bridging the long-standing "search-retrain gap" and distinguishing essential signals from noise without complex threshold tuning.

Extensive experiments across four benchmarks validate that ShuffleGate consistently outperforms state-of-the-art methods in both Feature and Dimension Selection tasks. It achieves a 15$\times$ speedup over permutation baselines and demonstrates extreme scalability by processing 270M parameters in just 700 seconds. Finally, in a top-tier industrial deployment, it compressed input dimensions by 10$\times$, yielding a 91\% increase in training throughput while serving billions of daily requests without performance degradation.

\end{abstract}

\ccsdesc[500]{Computing methodologies~Feature selection}
\ccsdesc[100]{Computing methodologies~Regularization}
\ccsdesc[300]{Information systems~Recommender systems}

\keywords{Feature Selection, Recommender systems, Dimension Pruning}


\maketitle
\section{Introduction}

Large-scale deep recommender systems (DRS) have become the backbone of modern online services, utilizing distinct features ranging from user demographics to interaction history to capture complex preferences \cite{DRS-survey, youtubeDRS}. However, the relentless accumulation of features has led to bloated models with billions of parameters, raising critical concerns regarding inference latency, storage costs, and overfitting risks \cite{LightRec}. To mitigate these issues, \textit{feature optimization} has emerged as a crucial research area. This typically encompasses three granularities: \textit{Feature Selection} (FS) to remove redundant fields \cite{ERASE}, \textit{Dimension Selection} (DS) to assign adaptive embedding sizes \cite{Autodim}, and even fine-grained \textit{Embedding Compression} (EC) to sparsify individual parameters within the embedding table \cite{embed-compression}.

In the realm of structure selection (FS and DS), existing methodologies primarily fall into three categories. Early approaches relied on heuristic proxies (e.g., Lasso, Random Forest) to estimate feature importance offline \cite{rf, gbdt}. With the rise of deep learning, the focus shifted to differentiable architecture search, which optimizes learnable masks for FS \cite{autofield, lpfs} or searches embedding sizes for DS \cite{Autodim, SSEDS} in an end-to-end manner. Alternatively, permutation-based methods (e.g., SHARK \cite{shark}) measure importance by observing prediction drops after feature corruption, prioritizing interpretability.

\textbf{Challenge.} However, existing solutions face three fundamental challenges in large-scale industrial application.
\textbf{First, Lack of Unification.} These tasks are typically studied in isolation with specialized algorithms (e.g., Autodim for DS \cite{Autodim}, AutoField for FS \cite{autofield}), which fragments the optimization process and increases engineering complexity.
\textbf{Second, The "Gate-Weight Entanglement".} Conventional differentiable methods \cite{lpfs, autofield, adafs} suffer from a pathological gradient interaction: as regularization pushes a gate toward zero, the optimizer inevitably \textit{inflates} the underlying feature weights to compensate for the signal loss \cite{polar_pruning}. This "seesaw" dynamic prevents gates from truly collapsing, resulting in continuous distributions where weak but valid signals overlap with noise. Consequently, practitioners are forced to rely on fragile threshold tuning, which lacks theoretical guarantees and often degrades pruning accuracy.
\textbf{Third, Computational Costs.} While permutation-based methods \cite{shark} offer precise sensitivity analysis, their $O(N)$ inference complexity becomes prohibitive as the search space expands. When facing massive candidate sets—whether for numerous feature fields or fine-grained dimension selection—such iterative evaluation renders these methods computationally intractable.


To address these challenges, we propose \textbf{ShuffleGate}, a unified framework that redefines importance estimation through the lens of \textit{sensitivity learning}. Instead of learning a simple weight, ShuffleGate asks: \textit{"If this component's information is destroyed (via random shuffling), how much does the prediction suffer?"} We implement this via a novel batch-wise shuffling operation integrated into a differentiable gating network. If a feature component is redundant, the model remains robust to its shuffling, and the gate naturally converges to zero to minimize the regularization penalty.

ShuffleGate represents a paradigm shift from post-hoc analysis to end-to-end learning. While iterative permutation-based methods like SHARK \cite{shark} establish a robust baseline for interpretability, their computational cost inherently scales with the search space.
For instance, evaluating the standard 39 features on Criteo via SHARK requires over 2 hours (7,492 seconds). In contrast, ShuffleGate leverages gradient-based sensitivity to evaluate all 270 million embedding entries within 700 seconds. This efficiency breakthrough enables us to scale beyond coarse-grained feature selection to massive parameter-level optimization—a regime previously unattainable for permutation paradigms.


Another key advantage of our approach is the natural polarization of importance scores.
ShuffleGate acts as  a rigorous noise filter. It separates importance scores into two distinct regions: redundant components collapse towards zero (e.g., $<10^{-3}$), while informative features—even weak ones—are preserved in a high-confidence range (significantly $>0.5$). 
As visualized in Figure \ref{fig:mechanism_polarization}, this creates a clear "safety margin" between redundant and preserved components \cite{polar_pruning}, ensuring that only truly irrelevant noise is removed \cite{lpfs}. 
Crucially, this polarization bridges the estimation-deployment gap: the validation AUC during gate learning becomes a reliable proxy for final post-pruning performance.

We demonstrate the effectiveness of ShuffleGate primarily on Feature  and Dimension Selection, where it consistently outperforms state-of-the-art baselines across
varying compression ratios. Beyond these standard tasks, we push the limit of our method by applying it to fine-grained embedding entry pruning. While practical serving of sparse embedding tables requires specific engineering support \cite{embed-compression}, this "embedding compression" experiment serves as a rigorous stress test for precision. It validates that ShuffleGate can effectively scale to millions of parameters, identifying and removing 99.9\% of redundant parameters on Criteo while maintaining competitive performance. Finally, the method has been deployed in a top-tier industrial video recommendation platform, successfully serving billions of daily requests.  

In summary, our contributions are:

\begin{itemize}
\item 
\textbf{Unified Sensitivity Learning:} We redefine structure search as a differentiable sensitivity learning task. By measuring model's sensitivity to stochastic information blockage, ShuffleGate offers intrinsic interpretability and achieves a qualitative leap in efficiency—transforming prohibitive iterative permutation analysis into a parallel, end-to-end training operation.
    \item \textbf{Reliable Polarization:} We demonstrate that ShuffleGate naturally separates signal from noise. This polarization mitigates the risk of  pruning weak features by creating a clear decision boundary, simplifying the selection process.
    \item \textbf{SOTA \& Scalability:} Experiments show ShuffleGate outperforms baselines on FS and DS tasks  across varying compression ratios. Furthermore, we demonstrate extreme scalability by evaluating 270 million parameters in minutes and reporting successful industrial deployment.
\end{itemize}

\begin{figure}[t]
    \centering
    \includegraphics[width=1.05\linewidth]{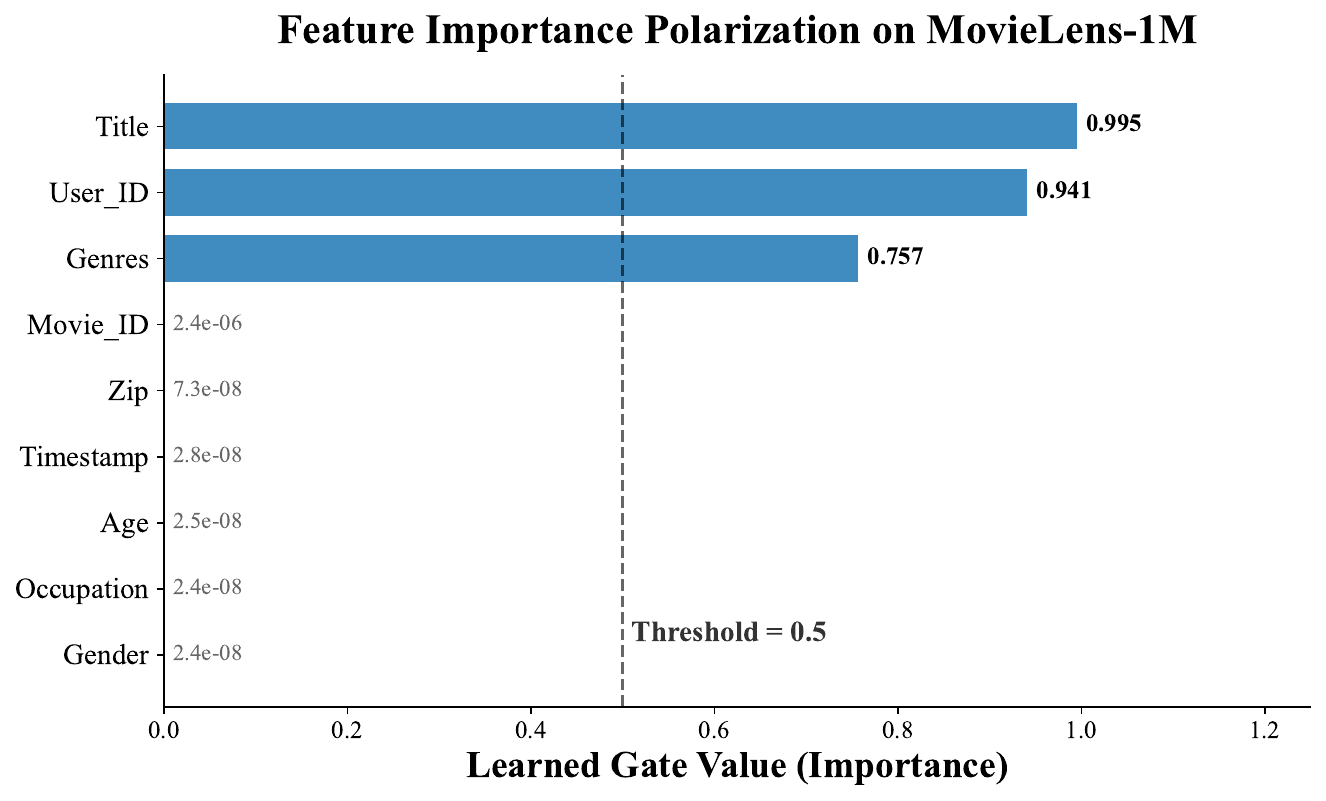}
    \caption{\textbf{Visualization of Polarization.} ShuffleGate learns a highly polarized distribution. Strong signals (blue) are kept high, while noise and redundant features are suppressed to near-zero, enabling a clear cut-off at the 0.5 threshold.}
    \label{fig:mechanism_polarization}
\end{figure}

\section{Related Work}
\label{sec:related_work}

\subsection{Feature Selection}
Feature selection (FS) is pivotal for deploying large-scale deep learning models in resource-constrained industrial environments.
Traditional FS methods often rely on statistical measures or tree-based models like Random Forest \cite{rf} and GBDT \cite{gbdt}. While interpretable, these approaches are generally model-agnostic and struggle to capture the complex, non-linear feature interactions inherent in deep learning models.

Current research focuses on differentiable mask-based methods, which integrate selection directly into deep model training. Representative methods like AutoField \cite{autofield} and LPFS \cite{lpfs} assign learnable gates to features and optimize them alongside model weights. However, these methods typically suffer from the coupling between masks and weights, where the optimizer inflates weights to compensate for shrinking masks \cite{polar_pruning}. This issue often prevents masks from truly collapsing to zero (lack of polarization) and causes inconsistency between the search and retrain stages. It is also worth noting that instance-wise gating approaches (e.g., AdaFS \cite{adafs}, LHUC \cite{lhuc}) dynamically re-weight features per sample. While effective for representation learning, they function as soft attention mechanisms rather than hard pruners, requiring the retention of all features during inference.

Another paradigm assesses importance by measuring the prediction sensitivity to feature corruption. Permutation Importance (PI) \cite{permutation} is the standard approach but is computationally prohibitive for high-dimensional inputs. SHARK \cite{shark} improves efficiency via Taylor approximation but still relies on an iterative "train-prune-retrain" framework. It is worth noting that interpretability methods like SHAP \cite{SHAP} also utilize feature perturbation. However, these methods focus on local interpretation for individual predictions, whereas ShuffleGate targets global importance for model pruning.

ShuffleGate unifies the strengths of both paradigms. By incorporating the shuffling mechanism into the differentiable training flow, we achieve simultaneous importance learning. The shuffled input decouples gates from weights, guaranteeing natural polarization and "One-Shot" efficiency without iterative retraining.

\subsection{Dimension Selection}
Beyond selecting entire feature fields, industrial applications often require finer-grained optimization, such as Dimension Selection (DS).

DS aims to assign varying embedding dimensions to different fields based on their importance. Existing DS methods like AutoDim \cite{Autodim} and SSEDS \cite{SSEDS} largely adapt the mask-based architecture (e.g., DARTS \cite{DARTS}) to the dimension level. Consequently, they inherit the same limitations of mask-based FS, such as the lack of polarization. ShuffleGate naturally extends to this task by treating each dimension chunk as a selectable unit, offering a robust solution for fine-grained structure search.

In the context of billion-scale recommender systems, DS serves as a form of pruning-based embedding compression. Unlike quantization \cite{FP16, Post4Bits,Int8/16,ALPT, MixedPrec} or hashing \cite{CompoEmb,ROBE,MemCom,BinaryCode,DoubleHash,DHE}, pruning physically removes redundant parameters. ShuffleGate belongs to this pruning category. It is orthogonal to quantization and hashing techniques and can be combined with them to achieve extreme compression rates.

\section{Preliminaries}
We consider a standard deep recommender system setup. Let $\mathcal{D} = \{(\mathbf{x}^{(j)}, y^{(j)})\}_{j=1}^N$ denote a dataset with $N$ samples. Each input $\mathbf{x}$ consists of $F$ feature fields. The embedding for field $i$ is $\mathbf{e}_i = \mathcal{E}_i[x_i]$, where $\mathcal{E}_i \in \mathbb{R}^{V_i \times d_i}$ is the embedding table. Concatenating all fields yields the dense input $\mathbf{e} = [\mathbf{e}_1, \dots, \mathbf{e}_F] \in \mathbb{R}^{D}$, where $D = \sum_{i=1}^F d_i$ represents the total input dimension. The model $f(\cdot; \Theta)$ takes $\mathbf{e}$ as input to generate predictions, and the parameters $\Theta$ are optimized to minimize the expected task-specific loss $\mathcal{L}_{\text{task}}$ (e.g., binary cross-entropy) over the dataset:
\begin{equation}
    \mathcal{J}(\Theta) = \mathbb{E}_{(\mathbf{x}, y) \sim \mathcal{D}} \left[ \mathcal{L}_{\text{task}}(y, f(\mathbf{e}; \Theta)) \right].
\end{equation}

\section{Methodology}

Evaluating feature importance via permutation (e.g., shuffling a feature and observing performance drop) is intuitive but computationally prohibitive for large-scale models \cite{permutation, shark}. To address this, we propose \textbf{ShuffleGate}, an end-to-end framework that integrates permutation-based sensitivity analysis directly into the training process.

\subsection{Unified Sensitivity Learning Framework}
\label{sec:unified_framework}

The core intuition of ShuffleGate is to measure the model's sensitivity to information loss at any granularity (field, dimension, or entry). Instead of physically removing a component (which alters the architecture) or masking it with zeros (which changes the distribution), we replace it with \textit{noise} drawn from its own marginal distribution.

Formally, let $\mathbf{z}$ denote a target component (e.g., a feature field vector or a scalar dimension) and $\tilde{\mathbf{z}}$ denote its shuffled noise counterpart. We introduce a learnable gate $g \in [0, 1]$ to control the information flow:
\begin{equation}
    \mathbf{z}^* = g \cdot \mathbf{z} + (1 - g) \cdot \operatorname{stopgrad}(\tilde{\mathbf{z}})
    \label{eq:unified_gate}
\end{equation}
where $\operatorname{stopgrad}(\cdot)$ prevents the model from learning to predict using the noise. The gate $g$ is parameterized by a learnable weight $\phi$ via sigmoid: $g = \sigma(\phi)$.

This mechanism forces the gate to learn the sensitivity: if $\mathbf{z}$ is redundant, replacing it with $\tilde{\mathbf{z}}$ causes negligible loss increase, and the regularization term will drive $g \to 0$. Conversely, if $\mathbf{z}$ is critical, the model must maintain $g \to 1$ to minimize the task loss. The unified objective is:
\begin{equation}
    \mathcal{L} = \mathcal{L}_{\text{task}} + \alpha \cdot \frac{1}{|\mathcal{S}|} \sum_{k \in \mathcal{S}} |g_k|
    \label{eq:unified_loss}
\end{equation}
where $\mathcal{S}$ is the set of all gates at the target granularity and $\alpha$ controls the sparsity pressure.  In incremental learning scenarios, we recommend a brief warm-up for newly introduced component before optimizing their gates to ensure fair competition with established components.

However, the practical feasibility of this framework hinges on generating the high-dimensional noise counterpart $\tilde{\mathbf{z}}$ efficiently online, which we address next.

\subsection{Scalable Noise Generation}
\label{sec:shuffle_op}

\begin{figure}[t]
    \centering
    \includegraphics[width=\linewidth]{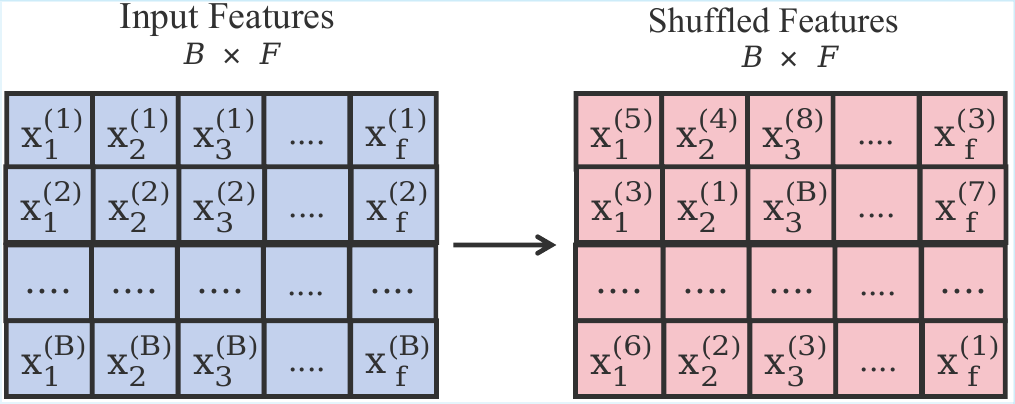}
    \caption{\textbf{Example of Batch-wise Shuffle Operation on Feature-Field Level.} Unlike global permutation, ShuffleGate permutes feature fields independently within a mini-batch.}
    \label{fig:shuffle_op}
\end{figure}

Efficiently generating the noise counterpart $\tilde{\mathbf{z}}$ is critical for end-to-end training. We introduce a \textbf{Batch-wise Shuffling} strategy that permutes data within the mini-batch to approximate the marginal distribution (See Figure \ref{fig:shuffle_op}).

\textbf{Mechanism.} Given a batch of input tensors $\mathbf{Z} \in \mathbb{R}^{B \times K}$, where $K$ represents the target granularity (i.e., $K=F$ for FS, $K=D$ for DS), we aim to break the feature-label correlation while preserving statistical properties. 
Instead of complex matrix multiplications, we implement this efficiently via a parallel \texttt{Gather} operation (Algorithm \ref{alg:efficient_shuffle}). Specifically, for each column $k \in \{1, \dots, K\}$, we generate a random permutation of indices along the batch dimension $B$ and reorder the elements independently.   With typically large batch sizes in industrial practice (e.g., $B \ge 8192$), this batch-wise strategy effectively approximates the global marginal distribution.

\textbf{Efficiency Analysis.} Unlike post-hoc permutation methods (e.g., SHARK) that require $O(N)$ inference passes for $N$ features, our approach integrates parallel sensitivity analysis directly into the forward pass. The complexity is $O(KB \log B)$, adding negligible overhead to the training loop.
This allows ShuffleGate to scale to millions of parameters, as demonstrated in our experiments.

Equipped with this scalable shuffling mechanism, we now demonstrate how to instantiate the unified framework across varying structural granularities.

\begin{algorithm}[t]
\small
\SetAlgoLined
\KwIn{$\mathbf{Z} \in \mathbb{R}^{B \times K}$: Input tensor}
\KwOut{$\tilde{\mathbf{Z}} \in \mathbb{R}^{B \times K}$: Shuffled tensor}
$\mathbf{R} \gets \text{Uniform}(0, 1)^{K \times B}$ \tcp*{Generate random keys}
$\boldsymbol{\pi} \gets \text{ArgSort}(\mathbf{R}, \text{dim}=1)$ \tcp*{Get permutation indices}
$\tilde{\mathbf{Z}} \gets \text{Gather}(\mathbf{Z}, \text{dim}=0, \text{index}=\boldsymbol{\pi}^T)$ \tcp*{Shuffle per column}
\Return{$\tilde{\mathbf{Z}}$}
\caption{Unified Batch-wise Shuffling}
\label{alg:efficient_shuffle}
\end{algorithm}

\subsection{Multi-Granularity Application}
\label{sec:applications}

The unified framework (Eq. \ref{eq:unified_gate}) can be seamlessly instantiated at three granularities.

\textbf{1. Feature Selection (Field-level).}
We apply gates to entire feature fields. Let $\mathbf{X}_i$ be the input indices for field $i$. The shuffled input $\tilde{\mathbf{X}}_i$ is generated by permuting the batch indices. The gated embedding is $\mathbf{e}_i^* = g_i \mathcal{E}_i[\mathbf{X}_i] + (1-g_i) \operatorname{stopgrad}(\mathcal{E}_i[\tilde{\mathbf{X}}_i])$, where scalar $g_i$ indicates field importance.

\textbf{2. Dimension Selection (Dimension-level).}
We apply gates to individual embedding dimensions. For the concatenated embedding matrix $\mathbf{E} \in \mathbb{R}^{B \times D}$, we generate $\tilde{\mathbf{E}}$ by shuffling each dimension column independently. The learned gate vector $\mathbf{g} \in \mathbb{R}^D$ then indicates  importance of each dimension.

\textbf{3. Micro-Granularity Stress Test (Entry-level).}
To evaluate the precision limits of ShuffleGate, we apply it to individual embedding entries. We assign a gate matrix $\mathbf{G}_i \in \mathbb{R}^{V_i \times d_i}$ to each embedding table. The active gate for a batch is retrieved via lookup: $\mathbf{g}_{\text{batch}} = \mathbf{G}_i[\mathbf{X}_i]$. 
 While serving sparse embedding tables requires specialized engineering \cite{embed-compression}, we utilize this setting primarily as a stress test. By attempting to prune millions of parameters individually (e.g., Criteo dataset), we verify whether ShuffleGate can distinguish signal from noise in extreme search spaces.

\subsection{Configuration and Pruning Strategies}
\label{sec:implementation}
A distinguishing feature of ShuffleGate is its natural \textbf{Polarization Effect}. Unlike traditional mask-gated based methods \cite{autofield,sfs,lpfs} which indiscriminately compress weights, ShuffleGate creates a "Discriminative Margin" where redundant gates collapse to zero while informative ones are pushed towards one.
Theoretical justifications for polarization effect are provided in
Appendix \ref{sec:theoretical_analysis}.
The polarization of gate values serves as a reliable indicator for hyperparameter configuration. In this section, we detail how to leverage this property to intuitively tune $\alpha$ and deploy pruning strategies tailored to varying industrial constraints.

\subsubsection{Efficient Hyperparameter Configuration}
\label{sec:tuning_efficiency}

Contrary to black-box hyperparameter optimization, tuning $\alpha$ in ShuffleGate is transparent and efficient due to two inherent properties:

   (1) \textbf{Immediate Feedback:} Engineers can monitor the mean gate value curve during early training steps. If the mean remains stuck at the initialization value (typically 1.0), it indicates $\alpha$ is too small to trigger regularization. One only needs to increase $\alpha$ until the curve starts to drop, confirming the mechanism is active.
    
    (2) \textbf{Monotonic Response:} The relationship between $\alpha$ and model sparsity is strictly monotonic. This turns hyperparameter search into a simple, deterministic direction-finding process rather than a random trial.

\subsubsection{The "WYSIWYG" Property}
As illustrated in Figure \ref{fig:wysiwyg}, due to the sharp polarization of gate values, the performance of the model equipped with the active ShuffleGate module (measured on validation data during training) closely mirrors its performance after physical pruning (Retrain AUC). 

This \textbf{"What-You-See-Is-What-You-Get" (WYSIWYG)} property eliminates the "search-retrain gap" commonly observed in conventional two-stage selection methods \cite{autofield, shark, sfs, Autodim, SSEDS, rf}. It allows practitioners to reliably assess the final pruned model quality by simply monitoring the Gate Learning AUC (i.e., the validation AUC of the gate-equipped model), avoiding expensive "train-prune-verify" cycles.

\subsubsection{Dual Pruning Strategies}
Leveraging the WYSIWYG property, we propose two pruning strategies based on different tuning logics for $\alpha$:

   (1) \textbf{Quality-First (Threshold-based):} Ideal for seeking the theoretically optimal model structure. Here, $\alpha$ acts as a sensitivity threshold relative to the loss. Tuning is intuitive: if the Gate Learning AUC drops (indicating useful signals are suppressed), decrease $\alpha$; if the AUC saturates but sparsity is insufficient (indicating noise remains), increase $\alpha$. Once balanced, the natural threshold of 0.5 effectively separates signal from noise.

    (2) \textbf{Budget-First (Rank-based):} Ideal for strict resource budgets (e.g., fixed FLOPs/memory). Here, the setting of $\alpha$ is highly forgiving. We do not need $\alpha$ to induce a specific sparsity level; rather, we only ensure it triggers broad polarization (e.g., anywhere between roughly 25\%--99\% of gates pushed below 0.5). Crucially, empirical observations show that the relative importance ranking remains consistent across varying $\alpha$ magnitudes. This eliminates the need for fine-grained grid search: as long as polarization is activated, the ranking is reliable. We can then directly truncate the Top-K features to satisfy precise compression targets (e.g., 50\% or 25\%) without retraining or re-searching. This strategy is adopted in our experiments to evaluate performance at fixed compression ratios.

\begin{figure}[t]
    \centering
    \includegraphics[width=0.9\linewidth]{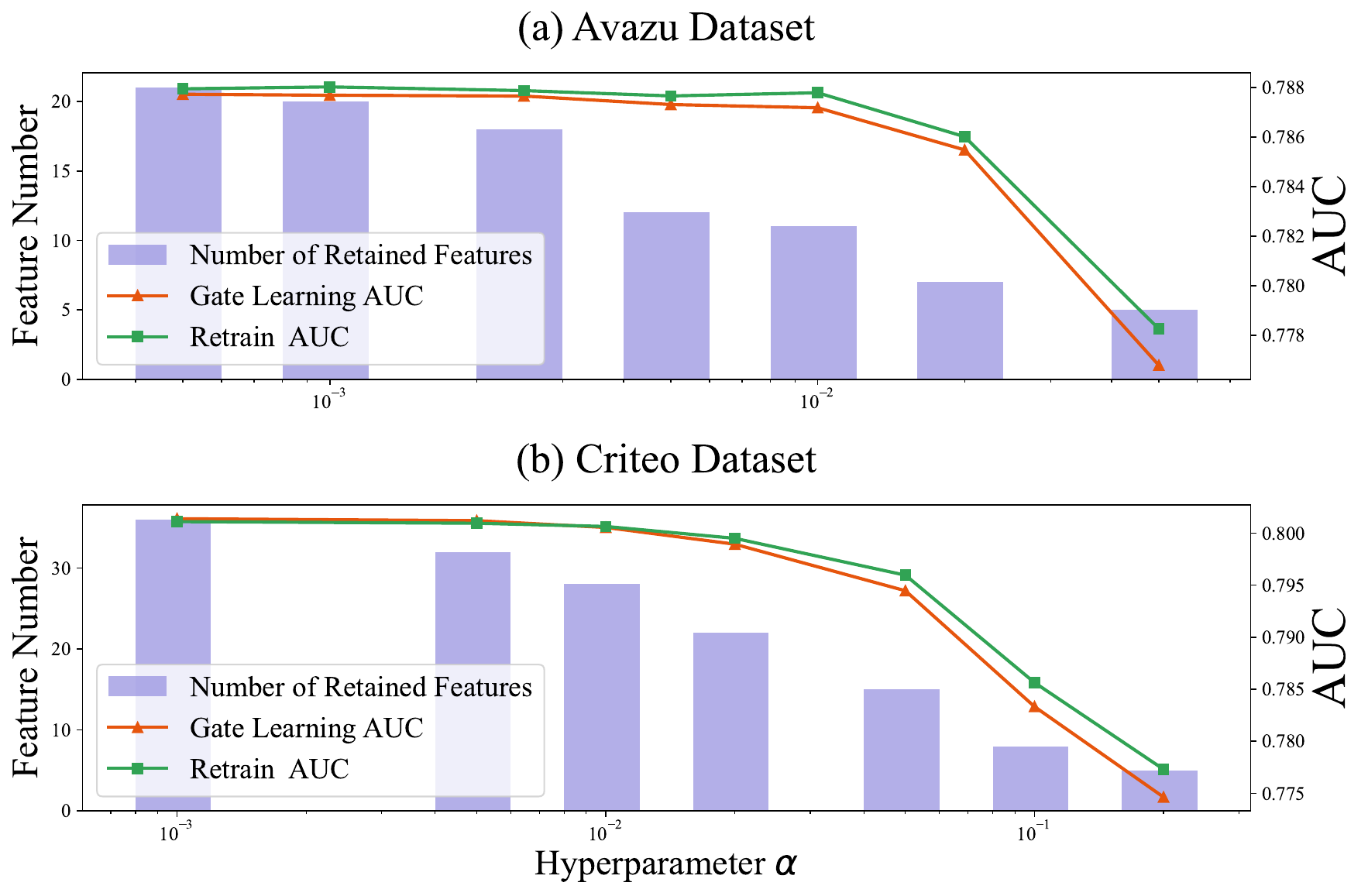} 
    \caption{\textbf{The WYSIWYG Property.} The AUC during the gate learning phase (\textbf{Gate Learning AUC}) exhibits a strong correlation with the AUC after pruning (Retrain AUC). This allows for reliable performance estimation without retraining.}
    \label{fig:wysiwyg}
\end{figure}

\subsubsection{Efficient Deployment}
Once the target sub-network is identified, ShuffleGate enables a highly efficient deployment pipeline:

(1) \textbf{Physical Pruning:} We physically reconstruct the model by removing redundant embedding columns and input nodes from the computation graph. This translates theoretical sparsity into real-world memory and latency reductions.
    
    (2) \textbf{Fine-tuning:} Unlike traditional methods that require retraining from scratch, the parameters retained by ShuffleGate are already jointly optimized. Therefore, the pruned model typically requires only a brief fine-tuning phase (warm start) to recover peak performance, significantly reducing the computational cost of model iteration.

\section{Experiments}
\label{sec:experiments}

We evaluate ShuffleGate on four public recommendation benchmarks to demonstrate its superiority over existing state-of-the-art methods. Our experiments are designed to answer four key research questions:
\begin{itemize}
    \item \textbf{RQ1 (Effectiveness):} Does ShuffleGate achieve state-of-the-art performance on standard Feature Selection (FS) and Dimension Selection (DS) tasks?
    \item \textbf{RQ2 (Scalability):} Can ShuffleGate efficiently handle massive-scale recommendation models in terms of both computational speed (vs. Permutation methods) and parameter scalability (extreme compression on 270M parameters)?
    \item \textbf{RQ3 (Industrial Application):} How does ShuffleGate perform in billion-scale industrial production environments?
\end{itemize}
We also provide a detailed case study to further illustrate the mechanism of the polarization effect in Appendix \ref{sec:mechanism}, which demonstrates that the polarization effect truly separates signal from noise without discarding weak features.

Our source code and supplementary PDF can be accessed at \url{https://anonymous.4open.science/r/Anonymous-ShuffleGate-F2DD/}.

\subsection{Experimental Setup}

\subsubsection{Datasets}
We utilize four widely used datasets: \textbf{Criteo}\footnote{\url{https://ailab.criteo.com/ressources/}}, \textbf{Avazu}\footnote{\url{https://www.kaggle.com/competitions/avazu-ctr-prediction}}, \textbf{MovieLens-1M}\footnote{\url{https://grouplens.org/datasets/movielens/1m/}}, and \textbf{AliCCP}\footnote{\url{https://tianchi.aliyun.com/dataset/408}}. Table \ref{tab:dataset} summarizes their statistics. We follow the data processing pipeline provided by ERASE \cite{ERASE}.

\begin{table}[h]
    \centering
    \caption{Dataset statistics.}
    \label{tab:dataset}
    \small
    \setlength{\tabcolsep}{4pt}
    \begin{tabular}{lcccc}
        \toprule
        \textbf{Dataset} & \textbf{Avazu} & \textbf{Criteo} & \textbf{ML-1M} & \textbf{AliCCP} \\
        \midrule
        \textbf{Samples} & 40,428,967 & 45,850,617 & 1,000,209 & 85,316,519 \\
        \textbf{Label} & Click & Click & Rating (1-5) & Click \\
        \textbf{Fields} & 23 & 39 & 9 & 23 \\
        \bottomrule
    \end{tabular}
\end{table}

\subsubsection{Evaluation Metrics}
We measure model performance using AUC. To aggregate results across datasets with varying difficulties, we adopt the Normalized AUC ($S_{\text{AUC}}$) metric following \cite{S_auc}:
\begin{equation}
S_\text{AUC}(A) = \frac{1}{|\Gamma|}\sum_{\mathcal{D} \in \Gamma} \frac{\text{AUC}(A,\mathcal{D})}{\max_{A' \in \mathcal{A}} \text{AUC}(A',\mathcal{D})}
\label{eq:s_auc}
\end{equation}
where $\Gamma$ is the set of datasets and $\mathcal{A}$ is the set of all baseline methods.
Compression efficiency is quantified by two metrics: \textbf{Feature Ratio (FR)}, which denotes the ratio of the number of retained feature fields to the total number of input fields; and \textbf{Dimension Ratio (DR)}, which denotes the ratio of the summed embedding dimensions after pruning to the original total embedding size.

\subsubsection{Baselines}
We compare ShuffleGate against comprehensive state-of-the-art methods:
\begin{itemize}
    \item \textbf{Feature Selection (FS):} (1) \textit{Heuristic}: Lasso \cite{lasso}, Random Forest (RF) \cite{rf}, XGBoost \cite{xgboost}; (2) \textit{Mask-based}: AutoField \cite{autofield}, LPFS \cite{lpfs}, SFS \cite{sfs}; (3) \textit{Permutation}: SHARK \cite{shark}.
    \item \textbf{Dimension Selection (DS):} (1) DimReg \cite{DimReg}; (2) SSEDS \cite{SSEDS}; (3) AutoDim \cite{Autodim}.
\end{itemize}

\noindent\textbf{Baseline Implementation Details.}
For FS baselines, we follow the optimized settings from ERASE \cite{ERASE}. For DS baselines, we reproduce them based on original papers: AutoDim uses \texttt{update\_frequency=10} and DARTS-style updates; DimReg sets the polarization weight $t=1.7$. We perform grid search for the rest hyperparameters where applicable to ensure fair comparison.

\subsubsection{Protocol \& Training Details}
We adopt a standard two-stage protocol \cite{ERASE} for all methods:
\begin{enumerate}
    \item \textbf{Search Stage:} The model is trained with the selection method to learn importance scores.
    \item \textbf{Retrain Stage:} We prune the model to target ratios and retrain from scratch.
\end{enumerate}
We evaluate two settings for Feature Selection: Feature Ratio, \textbf{FR $\in \{50\%, 25\%\}$}; and three settings for Dimension Selection: Dimension Ratio,  \textbf{DR $\in \{50\%, 25\%, 10\%\}$}. For the stress test (Embedding Compression), we explore an extreme ratio down to 0.1\%.

\noindent\textbf{Backbone Settings.} As noted in \cite{ERASE}, the choice of backbone does not significantly influence the results. In this case, all experiments use WideDeep \cite{widedeep} as the backbone with an embedding dimension of 8. We also provide additional results using the DCN \cite{dcn}  backbone in Appendix \ref{bb-study}. Optimization uses Adam (lr=0.001). We use an 8:1:1 split for training/validation/test and report averages over four runs to ensure stability. 

\noindent \textbf{ShuffleGate Configuration.} Adopting the Budget-First strategy (Sec. \ref{sec:implementation}), we tune $\alpha$ to induce sufficient polarization (approx. 50\% gates $<0.5$) and perform a one-shot search to truncate features or dimensions at various target ratios using the learned ranking. The specific $\alpha$ settings are: MovieLens-1M (0.1), AliCCP (0.001), Avazu (0.005), and Criteo (0.02).

\subsection{Main Results: Effectiveness (RQ1)}
\begin{table*}[t]
\centering
\caption{Feature Selection Results. Baselines are categorized by their underlying mechanism. Best results are highlighted in \textcolor{red}{\textbf{red}}. ShuffleGate consistently outperforms baselines across categories. On ML-1M, all top-performing methods converged to the same feature subset, resulting in identical AUC scores.}
\label{tab:fs_widedeep}
\small
\setlength{\tabcolsep}{3.5pt}
\begin{tabular}{l|c|cccc|c||cccc|c}
\toprule
\multirow{2}{*}{\textbf{Method}} & \multirow{2}{*}{\textbf{Category}} & \multicolumn{5}{c||}{\textbf{Target Retention Ratio (FR) $\approx$ 50\%}} & \multicolumn{5}{c}{\textbf{Target Retention Ratio (FR) $\approx$ 25\%}} \\
\cmidrule{3-12}
 & & \textbf{Criteo} & \textbf{Avazu} & \textbf{AliCCP} & \textbf{ML-1M} & $\mathbf{S_{\text{AUC}}}$ & \textbf{Criteo} & \textbf{Avazu} & \textbf{AliCCP} & \textbf{ML-1M} & $\mathbf{S_{\text{AUC}}}$ \\
\midrule
\textit{No\_Select} & - & 0.8014 & 0.7882 & 0.6598 & 0.7950 & 0.9954 & 0.8014 & 0.7882 & 0.6598 & 0.7950 & 0.9960 \\
\midrule
Lasso & Heuristic & 0.7462 & 0.7088 & 0.6054 & 0.6484 & 0.8872 & 0.7003 & 0.6016 & 0.5804 & 0.5286 & 0.7928 \\
XGBoost & Heuristic & 0.7682 & 0.7412 & 0.6499 & \textcolor{red}{\textbf{0.8097}} & 0.9710 & 0.7157 & 0.7091 & 0.5869 & 0.7338 & 0.8977 \\
RF & Heuristic & 0.7921 & 0.7869 & 0.6549 & 0.7924 & 0.9895 & 0.7626 & 0.7637 & 0.6032 & 0.6946 & 0.9237 \\
\midrule
SHARK & Permutation & 0.7974 & \textcolor{red}{\textbf{0.7870}} & 0.6582 & \textcolor{red}{\textbf{0.8097}} & 0.9977 & 0.7717 & \textcolor{red}{\textbf{0.7723}} & 0.6461 & \textcolor{red}{\textbf{0.8078}} & 0.9805 \\
\midrule
LPFS & Mask-based & 0.7944 & 0.7728 & 0.6554 & \textcolor{red}{\textbf{0.8097}} & 0.9913 & 0.7609 & 0.7714 & 0.6423 & \textcolor{red}{\textbf{0.8078}} & 0.9754 \\
SFS & Mask-based & 0.7969 & 0.7827 & 0.6579 & 0.7946 & 0.9915 & 0.7569 & 0.7685 & 0.6452 & \textcolor{red}{\textbf{0.8078}} & 0.9733 \\
AutoField & Mask-based & 0.7974 & 0.7869 & 0.6567 & 0.8077 & 0.9965 & 0.7738 & 0.7692 & 0.6485 & 0.8068 & 0.9808 \\
\midrule
\textbf{ShuffleGate} & \textbf{Ours} & \textcolor{red}{\textbf{0.7984}} & \textcolor{red}{\textbf{0.7870}} & \textcolor{red}{\textbf{0.6583}} & \textcolor{red}{\textbf{0.8097}} & \textcolor{red}{\textbf{0.9981}} & \textcolor{red}{\textbf{0.7841}} & \textcolor{red}{\textbf{0.7723}} & \textcolor{red}{\textbf{0.6489}} & \textcolor{red}{\textbf{0.8078}} & \textcolor{red}{\textbf{0.9854}} \\
\bottomrule
\end{tabular}
\end{table*}

\begin{table*}[t]
\centering
\caption{Dimension Selection Results. ShuffleGate demonstrates superior robustness, especially at extreme compression rates (Target DR=10\%), where competitors like DimReg degrade significantly. $^{\ddagger}$AutoDim lacks flexibility to achieve arbitrary target ratios, resulting in an average retention of $\approx$77\% in our experiments.}
\label{tab:ds_widedeep}
\resizebox{\textwidth}{!}{
\setlength{\tabcolsep}{3pt}
\begin{tabular}{l|cccc|c||cccc|c||cccc|c}
\toprule
\multirow{2}{*}{\textbf{Method}} & \multicolumn{5}{c||}{\textbf{Target DR $\approx$ 50\%}} & \multicolumn{5}{c||}{\textbf{Target DR $\approx$ 25\%}} & \multicolumn{5}{c}{\textbf{Target DR $\approx$ 10\%}} \\
\cmidrule{2-16}
 & \textbf{Criteo} & \textbf{Avazu} & \textbf{AliCCP} & \textbf{ML-1M} & $\mathbf{S_{\text{AUC}}}$ & \textbf{Criteo} & \textbf{Avazu} & \textbf{AliCCP} & \textbf{ML-1M} & $\mathbf{S_{\text{AUC}}}$ & \textbf{Criteo} & \textbf{Avazu} & \textbf{AliCCP} & \textbf{ML-1M} & $\mathbf{S_{\text{AUC}}}$ \\
\midrule
\textit{No\_Select} & 0.8014 & 0.7882 & 0.6598 & 0.7950 & 0.9975 & 0.8014 & 0.7882 & 0.6598 & 0.7950 & 0.9978 & 0.8014 & 0.7882 & 0.6598 & 0.7950 & 0.9954 \\
\midrule
AutoDim$^{\ddagger}$ & 0.8009 & \textcolor{red}{\textbf{0.7878}} & 0.6589 & 0.7943 & 0.9966 & - & - & - & - & - & - & - & - & - & - \\
DimReg & 0.8010 & 0.7876 & 0.6594 & 0.7943 & 0.9968 & 0.7996 & 0.7815 & 0.6524 & 0.8016 & 0.9944 & 0.7931 & 0.7448 & 0.6005 & 0.7983 & 0.9576 \\
SSEDS & 0.8006 & 0.7875 & 0.6590 & 0.7933 & 0.9962 & 0.7991 & 0.7856 & 0.6569 & 0.7923 & 0.9943 & 0.7880 & 0.7721 & 0.6487 & 0.7972 & 0.9826 \\
\textbf{ShuffleGate} & \textcolor{red}{\textbf{0.8011}} & \textcolor{red}{\textbf{0.7878}} & \textcolor{red}{\textbf{0.6595}} & \textcolor{red}{\textbf{0.8031}} & \textcolor{red}{\textbf{0.9996}} & \textcolor{red}{\textbf{0.8000}} & \textcolor{red}{\textbf{0.7868}} & \textcolor{red}{\textbf{0.6571}} & \textcolor{red}{\textbf{0.8021}} & \textcolor{red}{\textbf{0.9981}} & \textcolor{red}{\textbf{0.7946}} & \textcolor{red}{\textbf{0.7802}} & \textcolor{red}{\textbf{0.6543}} & \textcolor{red}{\textbf{0.8099}} & \textcolor{red}{\textbf{0.9932}} \\
\bottomrule
\end{tabular}
}
\end{table*}

We evaluate the performance of ShuffleGate against different categories of baselines on Feature Selection (FS) and Dimension Selection (DS).
For different retention ratios (e.g., 50\%, 25\%) in Tables \ref{tab:fs_widedeep} and \ref{tab:ds_widedeep}, ShuffleGate utilizes the same importance ranking derived from a single search process, truncating strictly based on the budget constraint.

\subsubsection{Feature Selection Performance}
Table \ref{tab:fs_widedeep} summarizes the FS performance on the WideDeep backbone under two retention ratios (FR=50\% and 25\%).

\noindent \textbf{Overall Dominance.} ShuffleGate achieves the highest aggregate $S_{\text{AUC}}$ across all settings, demonstrating superior generalization across diverse data distributions.

\textbf{Vs. Traditional Heuristics (Offline Proxies).} Methods like Lasso, RF, and XGBoost rely on shallow models (linear or tree-based) as offline proxies to estimate feature importance. However, deep recommender systems rely heavily on complex, non-linear embedding interactions that shallow proxies fail to capture. Consequently, these methods struggle on high-dimensional datasets like Criteo, lagging significantly behind end-to-end approaches.

\textbf{Vs. Mask-based Methods (Magnitude-based Fallacy).} 
While AutoField and LPFS are end-to-end, they fundamentally operate as \textit{magnitude-based} approaches—assuming that "smaller gate values imply less importance." We argue this assumption is flawed for deep models, where features with small weights can still be critical for non-linear interactions ("the butterfly effect"). ShuffleGate avoids this fallacy by measuring \textit{sensitivity} (loss degradation upon shuffling) rather than magnitude. This advantage is evident on the challenging Criteo dataset, where ShuffleGate significantly outperforms AutoField and LPFS.

\textbf{Vs. Permutation Method (Greedy vs. Joint).} 
Both SHARK and ShuffleGate belong to the \textit{Sensitivity-based} category, thus both achieving top-tier accuracy. However, ShuffleGate marginally outperforms SHARK on Criteo and AliCCP. We attribute this to the optimization strategy: SHARK evaluates features greedily (one-by-one), potentially missing coupled feature interactions, whereas ShuffleGate performs joint optimization of all gates during training, capturing holistic feature importance.

\textbf{"Less is More" on MovieLens.} On MovieLens-1M, ShuffleGate (AUC 0.8078) significantly outperforms the full-feature baseline (AUC 0.7950) even with 75\% features removed. This confirms that ShuffleGate accurately identifies and removes negative-gain noise features (e.g., timestamps) that lead to overfitting.

\subsubsection{Dimension Selection Performance}
We evaluate Dimension Selection (DS) under three reduction ratios: 50\%, 25\%, and 10\%. Table \ref{tab:ds_widedeep} presents the results.

\textbf{Robustness at Extreme Compression.} 
While most methods perform adequately at moderate compression, performance gaps widen significantly as resources become scarcer. At DR=10\%, ShuffleGate maintains a high $S_{\text{AUC}}$ of 0.9932. In contrast, DimReg relies on magnitude regularization, which fails to preserve critical dimensions under extreme constraints, leading to a performance collapse on AliCCP (AUC 0.6005 vs. ShuffleGate 0.6543).

\textbf{Comparison with AutoDim.}
AutoDim suffers from inflexible compression, resulting at a high retention ratio (DR $\approx$ 77\%) on average. Despite occupying significantly more memory resources than ShuffleGate (DR=50\%), it fails to demonstrate a performance advantage. ShuffleGate achieves a higher aggregate $S_{\text{AUC}}$ (0.9996 vs. 0.9966) and substantially outperforms AutoDim on ML-1M (80.31 vs. 79.43), proving superior parameter efficiency.

\subsubsection{Significance Test}
\label{appx:sig-test}
We perform one-sided paired Wilcoxon signed-rank tests to assess the statistical significance of our method's improvements over the baselines (Table \ref{tab:fs_widedeep} and \ref{tab:ds_widedeep}). The test results are presented in Table \ref{tab:sig_test}. All improvements by our method are shown to be
statistically significant.

\begin{table}[h]
\centering
\caption{ Test results of statistical significance.}
\label{tab:sig_test} 
\begin{tabular}{@{}llr@{}}
\toprule
\textbf{Task} & \textbf{Method} & \textbf{P-Value} \\ \midrule
\textbf{FS}   & LPFS            & 0.0              \\
              & RF              & 0.0              \\
              & SHARK           & 0.0111           \\
              & Autofield       & 0.0001           \\
              & Lasso           & 0.0              \\
              & Sfs             & 0.0              \\
              & XGB             & 0.0              \\ \midrule
\textbf{DS}   & Dim\_Reg        & 0.0005           \\
              & SSEDS           & 0.0001           \\
              & Autodim         & 0.0              \\ \midrule
\end{tabular}
\end{table}

\subsection{Scalability Analysis (RQ2)}
\label{sec:scalability}

In this section, we evaluate whether ShuffleGate can handle industrial-scale recommendation models (e.g., Criteo with 270M parameters) in terms of both computational speed and parameter efficiency.

\subsubsection{The Efficiency Comparison}

We first compare the search time of ShuffleGate against SHARK, the standard permutation-based method. Figure \ref{fig:efficiency} reports the time cost on Criteo.

\textbf{The Efficiency Gap.} On the Feature Selection task (39 features), SHARK requires \textbf{7,492 seconds} due to its sequential inference strategy. In contrast, ShuffleGate completes the task in \textbf{501 seconds}, achieving a \textbf{15$\times$ speedup}. This advantage stems from our batch-wise shuffling mechanism, which allows joint optimization of all importance gates in a single forward-backward pass, whereas SHARK must evaluate each feature greedily.

\begin{figure}[t]
    \centering
    \includegraphics[width=\linewidth]{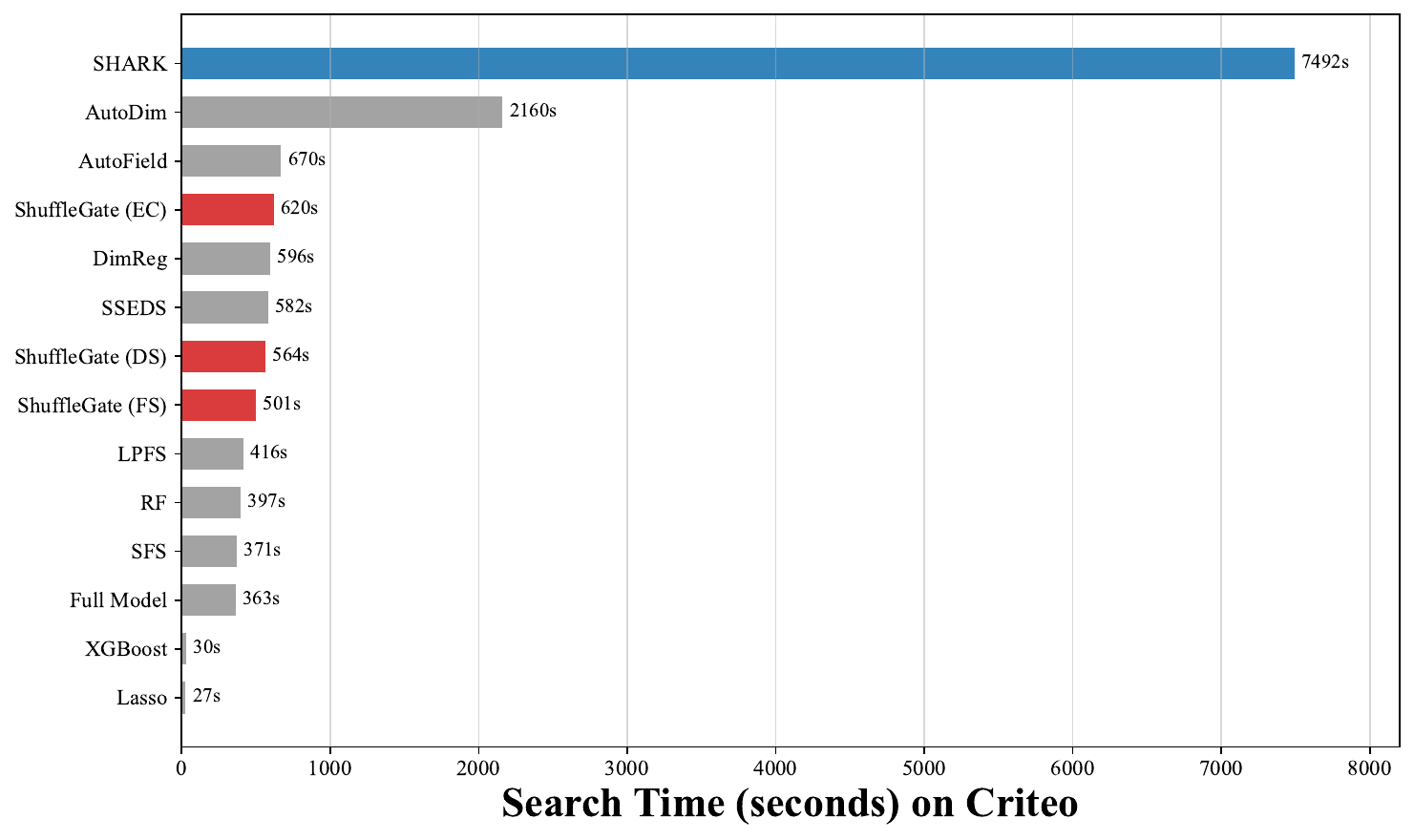}
    \caption{\textbf{Search Time Efficiency on Criteo.} ShuffleGate achieves a 15$\times$ speedup over SHARK on feature selection (39 features). More importantly, its time cost remains constant ($O(1)$) even when scaling to 270 million embedding entries (ShuffleGate-EC), whereas SHARK would be computationally infeasible.}
    \label{fig:efficiency}
\end{figure}

\subsubsection{Parameter Scalability: The Criteo Stress Test}
To further validate scalability, we conduct a "Stress Test" on the Criteo dataset. The gap becomes qualitative when scaling to Embedding Compression (EC), where the search space explodes from 39 features to \textbf{270 million} embedding entries.
Since SHARK's complexity scales linearly with the number of components ($O(N)$), applying it to this task would theoretically take:
\begin{equation}
    7,492 \text{ s} \times \frac{270,000,000}{39} \approx 5.2 \times 10^{10} \text{ s} \approx \textbf{1,600 years}
\end{equation}
This makes traditional permutation methods computationally infeasible for fine-grained compression.
Conversely, ShuffleGate demonstrates $O(1)$ scalability with respect to component count. As shown in Figure \ref{fig:efficiency}, evaluating 270M parameters (ShuffleGate-EC) takes only \textbf{620 seconds} ($\approx$ 10 minutes). ShuffleGate thus transforms a task that would take nearly two millennia into one feasible within minutes.

In terms of  effectiveness, ShuffleGate compresses the embedding table to an extreme density of 0.1\% (i.e., removing 99.9\% of parameters). For this task, we adjust $\alpha$ to 2.0 to maintain the relative regularization magnitude, as the averaging denominator in our loss function expands from 39 features to 270 million embedding parameters.

We do not include comparisons with other embedding compression baselines here, as the primary objective of this experiment is not to achieve state-of-the-art performance in the compression task itself. Instead, it serves as an isolated stress test to demonstrate that our method remains stable and scalable even under extreme large scale.

\begin{table}[h]
    \centering
    \caption{\textbf{Criteo Stress Test (Embedding Compression).} Even with 99.9\% of parameters removed, ShuffleGate achieves a higher AUC than the full model. }
    \label{tab:criteo_stress}
    \small
    \setlength{\tabcolsep}{6pt}
    \begin{tabular}{lccc}
        \toprule
        \textbf{Method} & \textbf{Params Kept} & \textbf{Sparsity} & \textbf{AUC} \\
        \midrule
        Full Model & 100\% & 0\% & 0.8014 \\
        \textbf{ShuffleGate (Ours)} & \textbf{0.1\%} & \textbf{99.9\%} & \textcolor{red}{\textbf{0.8027}} \\
        \bottomrule
    \end{tabular}
\end{table}

As shown in Table \ref{tab:criteo_stress}, ShuffleGate maintains an AUC of \textbf{0.8027}, which is surprisingly \textit{higher} than the full model's AUC (0.8014).
This counter-intuitive result highlights two key findings:
\begin{enumerate}
    \item \textbf{High Redundancy:} Industrial-scale embedding tables contain massive amounts of noise and redundant parameters that do not contribute to (and may even harm) generalization.
    \item \textbf{Precision:} ShuffleGate's sensitivity learning is precise enough to locate important components within a hundred-million-scale parameter space, proving its capability for  industrial scale.
\end{enumerate}

\subsection{Industrial Deployment and Impact (RQ3)}
\label{sec:industrial}

To validate ShuffleGate's adaptability, we deployed it in Bilibili's billion-scale production environment. We report results from two representative scenarios. It is worth noting that while ShuffleGate simultaneously reduces both training and inference resource consumption in both cases, we highlight the specific primary bottleneck that was most critical for each scenario.

\subsubsection{Scenario A: Inference-Resource Heavy}
This scenario involves an online ranking service where the primary constraint was deployment cost and inference latency.
\begin{itemize}
    \item \textbf{Strategy (Feature Deletion):} Applying ShuffleGate solely for feature selection, we identified and physically removed \textbf{60\%+} of redundant features.
    \item \textbf{Impact:} This aggressive "slimming" reduced the model parameter size by \textbf{40\%+}, significantly cutting down computational overhead (FLOPs) for both training and inference. Specifically, it yielded a \textbf{20\%+} boost in online inference speed.
    \item \textbf{Metric Impact:} During a six-month online A/B test, core business metrics (e.g., CTR) remained comparable to the baseline. The metric fluctuations were strictly within 0.1\%, which falls within the range of random fluctuation and indicates no statistically significant degradation.
\end{itemize}

\subsubsection{Scenario B: Training-IO Heavy}
This scenario involves a large-scale recommendation model where the training bottleneck was the CPU-to-GPU bandwidth caused by the massive dimension of the concatenated embedding layer.
\begin{itemize}
    \item \textbf{Strategy (Deletion + Dim. Reduction):} We adopted a dual-compression strategy. Beyond removing non-predictive features, we further compressed the embedding dimensions of retained features based on their ShuffleGate importance scores (reducing them from 64 to 32, 16, or 8).
    \item \textbf{Impact:} This strategy reduced the dimension of the Concatenated Embedding Layer (the input to the MLP) from over \textbf{10,000} to approximately \textbf{1,000}. This 10$\times$ reduction in the input vector effectively unclogged the IO bandwidth. Consequently, the training speed surged by \textbf{91\%}, while inference latency was also reduced due to the smaller input size.
    \item \textbf{Metric Impact:} Despite the significant reduction in information capacity, both offline AUC and online interaction metrics remained comparable within six months. Performance gaps were consistently less than 0.1\%, confirming that the compressed model preserved the predictive power of the original large model.
\end{itemize}

\begin{table}[h]
    \centering
    \caption{\textbf{Industrial Application Summary.} In both cases, computational resources are significantly saved while ensuring comparable business metrics (diff $<$ 0.1\%).}
    \label{tab:industrial_impact}
    \resizebox{\columnwidth}{!}{
    \setlength{\tabcolsep}{4pt}
    \begin{tabular}{l|c|c}
        \toprule
        \textbf{Dimension} & \textbf{Scenario A} & \textbf{Scenario B} \\
        \midrule
        \textit{Primary Bottleneck} & Inference Compute \& Memory & Training IO \& Bandwidth \\
        Strategy & Feature Deletion & Deletion + Dim. Reduction \\
        \textit{Dim. Policy} & Keep Original & \textbf{64 $\to$ 32 / 16 / 8} \\
        \midrule
        Optimization & \textbf{60\%+} Features Removed & Input Dim: \textbf{10k+ $\to$ 1k+} \\
        Primary Gain & Inference Speed $\uparrow$ \textbf{20\%+} & Training Speed $\uparrow$ \textbf{91\%} \\
        \midrule
        Metrics & Comparable (Diff $<$ 0.1\%) & Comparable (Diff $<$ 0.1\%) \\
        \bottomrule
    \end{tabular}
    }
\end{table}

\section{Conclusion}
\label{sec:conclusion}

In this paper, we proposed ShuffleGate, a unified and efficient structure search framework designed for large-scale industrial deep models. Capable of handling both coarse-grained Feature Selection and fine-grained Dimension Selection, ShuffleGate fundamentally addresses the limitations of existing methods. By incorporating the shuffling mechanism into the differentiable training flow, we effectively decouple the gradient interaction between gates and weights. This design guarantees theoretical Self-Polarization, enabling the automatic separation of signal from noise without complex heuristics.

From an engineering perspective, ShuffleGate bridges the long-standing "Search-Retrain Gap." Its unique \textbf{"What-You-See-Is-What-You-Get" (WYSIWYG)} property, combined with a single hyperparameter design ($\alpha$), transforms structure search from a "black-box" trial into a transparent and controllable process. This allows practitioners to reliably estimate final model performance during the search stage, eliminating the need for expensive iterative retraining.

Extensive deployment in Bilibili's billion-scale production environment has demonstrated ShuffleGate's robust adaptability across distinct resource bottlenecks. In inference-constrained scenarios, it served as a feature selector, reducing model size by over \textbf{40\%} via aggressive pruning. In IO-constrained training scenarios, it functioned as a dimension selector, compressing embedding dimensions (e.g., from 64 to 8/16) to break bandwidth limits, achieving a massive \textbf{91\%} increase in training throughput. In all cases, ShuffleGate achieved these significant resource savings while maintaining comparable business metrics, establishing it as a highly effective and reliable solution for industrial model compression.

\bibliographystyle{ACM-Reference-Format}
\bibliography{sample-base}

\appendix
\newpage
\clearpage
\onecolumn
\section{Theoretical Analysis: Mechanism of Self-Polarization}
\label{sec:theoretical_analysis}

A distinguishing feature of ShuffleGate is its natural \textbf{Polarization Effect}. Unlike traditional mask-gate based methods which indiscriminately compress weights, ShuffleGate creates a "Discriminative Margin" where redundant gates collapse to zero while informative ones are pushed towards one. In this section, we analyze the gradient dynamics on the \textit{expected loss} to provide a rigorous guarantee.

\subsubsection{Gradient Dynamics Setup}
Consider the optimization of a single gate $g_i \in [0, 1]$ for feature $i$. Let $J(g_i) = \mathbb{E}_{\mathbf{x} \sim \mathcal{D}}[\mathcal{L}_{\text{task}}(f(\mathbf{x}, g_i))]$ denote the expected task loss over the data distribution $\mathcal{D}$. The total objective function is:
\begin{equation}
    \mathcal{J}_{\text{total}}(g_i) = J(g_i) + \alpha g_i
\end{equation}
The gradient descent update with learning rate $\eta$ is given by:
\begin{equation}
    g_i \leftarrow g_i - \eta \left( \frac{\partial J(g_i)}{\partial g_i} + \alpha \right)
\end{equation}
Since $g_i \ge 0$, the regularization term provides a constant positive pressure $+\alpha$. The convergence direction of $g_i$ is determined by the magnitude of the task gradient $\frac{\partial J}{\partial g_i}$ relative to $\alpha$.

\subsubsection{Noise Suppression Guarantee}
We first define "non-predictive" features from a probabilistic perspective. Intuitively, a feature is non-predictive if its contribution to the loss variation (measured by the gradient) is bounded by the scale of its embedding perturbation.

\begin{definition}[$\epsilon$-Non-Predictive Feature]
\label{def:non_predictive}
Let $\Delta_i = \mathbb{E}[\| \mathbf{e}_i - \tilde{\mathbf{e}}_i \|]$ be the expected embedding distance caused by shuffling. Feature $i$ is $\epsilon$-non-predictive if the gradient of the expected task loss satisfies the following condition:
\begin{equation}
    \left| \frac{\partial J(g_i)}{\partial g_i} \right| \le \epsilon \cdot \Delta_i, \quad \forall g_i \in [0, 1]
\end{equation}
Here, $\epsilon$ represents the sensitivity coefficient of the model to the feature's noise.
\end{definition}

\begin{theorem}[Noise Suppression]
\label{thm:noise_suppression}
For an $\epsilon$-non-predictive feature, if the regularization coefficient satisfies $\alpha > \epsilon \Delta_i$, then the total gradient is strictly positive, driving $g_i$ to 0.
\end{theorem}

\begin{proof}
The proof follows directly from the gradient composition:
\begin{enumerate}
    \item \textbf{Gradient Bound:} From Definition \ref{def:non_predictive}, the negative task gradient is lower-bounded by:
    \begin{equation}
        \frac{\partial J(g_i)}{\partial g_i} \ge - \left| \frac{\partial J(g_i)}{\partial g_i} \right| \ge - \epsilon \Delta_i
    \end{equation}
    \item \textbf{Total Gradient Construction:} Substituting this into the total gradient expression:
    \begin{equation}
        \frac{\partial \mathcal{J}_{\text{total}}}{\partial g_i} = \frac{\partial J(g_i)}{\partial g_i} + \alpha \ge -\epsilon \Delta_i + \alpha
    \end{equation}
    \item \textbf{Conclusion:} Since we set $\alpha > \epsilon \Delta_i$, it follows that:
    \begin{equation}
        \frac{\partial \mathcal{J}_{\text{total}}}{\partial g_i} > 0
    \end{equation}
\end{enumerate}
Consequently, the optimizer will consistently decrease $g_i$ until it hits the lower bound 0.
\end{proof}

\subsubsection{Signal Preservation Guarantee}
Conversely, for a predictive feature, shuffling destroys valid information, causing a significant increase in the expected loss. We show that if this loss gap exceeds $\alpha$, the gate is preserved.

\begin{theorem}[Signal Preservation]
\label{thm:signal_preservation}
Assume the expected task loss $J(g_i)$ is convex w.r.t. $g_i$. Let $\Delta J = J(0) - J(1)$ be the expected signal strength (loss increase due to shuffling). If $\Delta J > \alpha$, then the total gradient is negative, driving $g_i$ to 1.
\end{theorem}

\begin{proof}
We utilize the first-order property of convex functions.
\begin{enumerate}
    \item \textbf{Convexity Inequality:} Since $J(g_i)$ is convex, for any $g_i \in [0, 1]$, the curve lies below its chord connecting endpoints $g_i=0$ and $g_i=1$:
    \begin{equation}
        J(g_i) \le g_i J(1) + (1-g_i) J(0)
    \end{equation}
    \item \textbf{Derivative Upper Bound:} Rearranging the inequality terms:
    \begin{equation}
        J(g_i) - J(0) \le g_i (J(1) - J(0))
    \end{equation}
    Differentiating the convex inequality w.r.t $g_i$, implies that the gradient is upper-bounded by the slope of the chord:
    \begin{equation}
        \frac{\partial J(g_i)}{\partial g_i} \le J(1) - J(0) = - (J(0) - J(1)) = -\Delta J
    \end{equation}
    \item \textbf{Total Gradient \& Conclusion:} Adding the regularization term $\alpha$:
    \begin{equation}
        \frac{\partial \mathcal{J}_{\text{total}}}{\partial g_i} = \frac{\partial J(g_i)}{\partial g_i} + \alpha \le -\Delta J + \alpha
    \end{equation}
    Since the signal strength satisfies $\Delta J > \alpha$, we have $\frac{\partial \mathcal{J}_{\text{total}}}{\partial g_i} < 0$. Thus, the optimizer will increase $g_i$, pushing it towards the upper bound 1.
\end{enumerate}
\end{proof}

\newpage

\section{Experiment Supplement}

\subsection{Backbone Study}
\label{bb-study}
For the Feature Selection experiemnt, we provide the
retrain auc results based on DCN , which is available in Table \ref{tab:fs-full-1} and \ref{tab:fs-full-2}.

\begin{table*}
\centering
\caption{Feature Selection Results with DCN Backbone (Field Retention Ratio = 50\%). Best AUC results among feature selection methods are highlighted in \textcolor{red}{\textbf{red}}. $^{\dagger}$Reference baseline (not included in comparison).}
\label{tab:fs-full-1}
\setlength{\tabcolsep}{4pt}
\begin{tabular}{l|cc|cc|cc|cc}
\toprule
\multirow{2}{*}{\textbf{Method}} & \multicolumn{2}{c|}{\textbf{Criteo}} & \multicolumn{2}{c|}{\textbf{Avazu}} & \multicolumn{2}{c|}{\textbf{Aliccp}} & \multicolumn{2}{c}{\textbf{MovieLens-1M}} \\
& AUC & FR & AUC & FR & AUC & FR & AUC & FR \\
\midrule
\textit{No\_Select}$^{\dagger}$ & \textit{0.8017} & 1.000 & \textit{0.7884} & 1.000 & \textit{0.6574} & 1.000 & \textit{0.7949} & 1.000 \\
& & & & & & & & \\[-0.5em]
ShuffleGate & \textcolor{red}{\textbf{0.7991}} & 0.487 & \textcolor{red}{\textbf{0.7880}} & 0.500 & \textcolor{red}{\textbf{0.6576}} & 0.478 & \textcolor{red}{\textbf{0.8102}} & 0.444 \\
AutoField & 0.7983 & 0.487 & 0.7879 & 0.500 & 0.6574 & 0.478 & 0.8090 & 0.444 \\
SHARK & 0.7975 & 0.487 & 0.7878 & 0.500 & 0.6571 & 0.478 & \textcolor{red}{\textbf{0.8102}} & 0.444 \\
SFS & 0.7969 & 0.487 & 0.7857 & 0.500 & 0.6563 & 0.478 & 0.8090 & 0.444 \\
LPFS & 0.7970 & 0.487 & 0.7806 & 0.500 & 0.6495 & 0.478 & 0.8020 & 0.444 \\
Random Forest & 0.7932 & 0.487 & 0.7878 & 0.500 & 0.6556 & 0.478 & 0.7949 & 0.444 \\
XGBoost & 0.7686 & 0.487 & 0.7414 & 0.500 & 0.6507 & 0.478 & \textcolor{red}{\textbf{0.8102}} & 0.444 \\
Lasso & 0.7465 & 0.487 & 0.7092 & 0.500 & 0.6046 & 0.478 & 0.6476 & 0.444 \\
\bottomrule
\end{tabular}
\end{table*}

\begin{table*}
\centering
\caption{Feature Selection Results with DCN Backbone (Field Retention Ratio = 25\%). Best AUC results among feature selection methods are highlighted in \textcolor{red}{\textbf{red}}. $^{\dagger}$Reference baseline (not included in comparison).}
\label{tab:fs-full-2}
\setlength{\tabcolsep}{4pt}
\begin{tabular}{l|cc|cc|cc|cc}
\toprule
\multirow{2}{*}{\textbf{Method}} & \multicolumn{2}{c|}{\textbf{Criteo}} & \multicolumn{2}{c|}{\textbf{Avazu}} & \multicolumn{2}{c|}{\textbf{Aliccp}} & \multicolumn{2}{c}{\textbf{MovieLens-1M}} \\
& AUC & FR & AUC & FR & AUC & FR & AUC & FR \\
\midrule
\textit{No\_Select}$^{\dagger}$ & \textit{0.8017} & 1.000 & \textit{0.7884} & 1.000 & \textit{0.6574} & 1.000 & \textit{0.7949} & 1.000 \\
& & & & & & & & \\[-0.5em]
ShuffleGate & \textcolor{red}{\textbf{0.7857}} & 0.231 & \textcolor{red}{\textbf{0.7742}} & 0.227 & \textcolor{red}{\textbf{0.6496}} & 0.217 & \textcolor{red}{\textbf{0.8077}} & 0.222 \\
AutoField & 0.7766 & 0.231 & 0.7731 & 0.227 & 0.6472 & 0.217 & 0.8076 & 0.222 \\
SHARK & 0.7696 & 0.231 & 0.7731 & 0.227 & 0.6462 & 0.217 & 0.8076 & 0.222 \\
LPFS & 0.7690 & 0.231 & 0.7713 & 0.227 & 0.6167 & 0.217 & 0.8076 & 0.222 \\
SFS & 0.7600 & 0.231 & 0.7724 & 0.227 & 0.6452 & 0.217 & 0.8076 & 0.222 \\
Random Forest & 0.7639 & 0.231 & 0.7646 & 0.227 & 0.6049 & 0.217 & 0.6974 & 0.222 \\
XGBoost & 0.7157 & 0.231 & 0.7093 & 0.227 & 0.5867 & 0.217 & 0.7337 & 0.222 \\
Lasso & 0.7005 & 0.231 & 0.6016 & 0.227 & 0.5801 & 0.217 & 0.5311 & 0.222 \\
\bottomrule
\end{tabular}
\end{table*}

\subsection{Mechanism Analysis}
\label{sec:mechanism}

To demystify the internal decision-making process of ShuffleGate, we conduct an in-depth case study on the MovieLens-1M dataset. For this analysis, we adopt $\alpha=0.03$, a setting selected based on validation performance to strike an optimal balance between sufficient noise filtration and signal preservation.

\subsubsection{Visualizing Polarization}
Figure \ref{fig:mechanism_polarization} visualizes the learned gate values for all features. ShuffleGate exhibits a strong \textbf{polarization effect}, creating a sharp dichotomy between useful and useless features:
\begin{itemize}
    \item \textbf{Selected Zone:} Features like \texttt{Title} and \texttt{User\_ID} are assigned gates close to 1.0, while \texttt{Genres} is safely maintained above the threshold (0.5).
    \item \textbf{Suppressed Zone:} Irrelevant or redundant features (e.g., \texttt{Movie\_ID}, \texttt{Timestamp}) are aggressively compressed to near-zero values ($< 10^{-6}$).
\end{itemize}
This "black-and-white" distribution eliminates the ambiguity often found in traditional soft-gating methods, enabling decisive pruning.

\subsubsection{Stepwise Validation and Insights}
To validate whether the polarization aligns with true feature utility, we perform a two-phase validation as shown in Table \ref{tab:mechanism_validation}. First, we cumulatively add "Selected" features. Second, we individually add "Suppressed" features back to the peak model to verify their redundancy.

\begin{table}[t]
    \centering
    \caption{\textbf{Mechanism Validation.} We first build the model using high-gate features (Steps 1-3), reaching peak performance. We then verify the "suppressed" features (Steps 4-9) by adding them individually to the peak model. Results show they are either redundant, negligible, or harmful.}
    \label{tab:mechanism_validation}
    \resizebox{0.5\columnwidth}{!}{
    \setlength{\tabcolsep}{4pt}
    \begin{tabular}{clccc} 
        \toprule
        \textbf{Step} & \textbf{Feature Set} & \textbf{Gate} & \textbf{AUC} & \textbf{$\Delta$AUC} \\
        \midrule
        \multicolumn{5}{l}{\textit{\textbf{Part 1: Construction (High Gate > 0.5)}}} \\
        \cmidrule(r){1-5} 
        1 & [Title] & 0.995 & 0.7336 & - \\
        2 & + User\_ID & 0.941 & 0.8073 & \textcolor{darkgreen}{+0.0737} \\ 
        \textbf{3} & \textbf{+ Genres (Peak Model)} & \textbf{0.757} & \textcolor{red}{\textbf{0.8105}} & \textcolor{darkgreen}{\textbf{+0.0032}} \\
        \midrule
        \multicolumn{5}{l}{\textit{\textbf{Part 2: Verification (Base = Peak Model + Feature X)}}} \\
        \cmidrule(r){1-5}
        4 & + Movie\_ID & $\approx 0.0$ & 0.8099 & \textcolor{gray}{-0.0006} \\
        5 & + Zip & $\approx 0.0$ & 0.8105 & \textcolor{gray}{0.0000} \\
        6 & + Age & $\approx 0.0$ & 0.8106 & +0.0001 \\
        7 & + Occupation & $\approx 0.0$ & 0.8106 & +0.0001 \\
        8 & + Gender & $\approx 0.0$ & 0.8107 & +0.0002 \\
        9 & + Timestamp & $\approx 0.0$ & \textbf{0.7938} & \textcolor{blue}{\textbf{-0.0167}} \\
        \bottomrule
    \end{tabular}
    }
\end{table}

\textbf{1. The "Safety Margin" for Weak Signals.}
A key strength of ShuffleGate is preserving weak but informative features. \texttt{Genres} (Gate 0.757) is not as dominant as \texttt{Title}, but adding it improves AUC from 0.8073 to \textbf{0.8105}. This confirms that ShuffleGate successfully establishes a "safety margin," retaining weak signals that contribute to generalization.

\textbf{2. Conditional Redundancy.}
Adding \texttt{Movie\_ID} (Gate $\approx 0$) back to the peak model drops the AUC slightly to 0.8099. This confirms that the information in \texttt{Movie\_ID} is fully redundant given \texttt{Title}.

\textbf{3. Pruning Negligible Features (Occam's Razor).}
Features like \texttt{Zip}, \texttt{Age}, \texttt{Occupation}, and \texttt{Gender} are all suppressed to near-zero. As shown in Table \ref{tab:mechanism_validation} (Steps 5-8), adding them back yields negligible AUC fluctuations ($0.0000 \sim +0.0002$). While technically neutral, keeping them incurs memory costs without meaningful performance gains. ShuffleGate's sensitivity objective effectively applies Occam's Razor, trimming these statistically insignificant parameters to maximize efficiency.

\textbf{4. Filtering Harmful Noise.}
Crucially, adding \texttt{Timestamp} (Gate $\approx 0$) causes a significant performance drop (AUC $0.8105 \to 0.7938$). This identifies \texttt{Timestamp} as harmful noise (negative transfer), which ShuffleGate correctly filters out to prevent overfitting.

\end{document}